\documentclass{article}

\usepackage[round]{natbib}
\usepackage{fullpage}

\usepackage[utf8]{inputenc} 
\usepackage[T1]{fontenc}    
\usepackage{hyperref}       
\usepackage{url}            
\usepackage{booktabs}       
\usepackage{amsfonts}       
\usepackage{nicefrac}       
\usepackage{microtype}      
\usepackage{xcolor}         

\usepackage{float}
\usepackage{wrapfig}

\usepackage{graphicx}
\usepackage{subcaption}
\usepackage{bm}
\usepackage{booktabs} 
\usepackage{framed}
\usepackage[inline,shortlabels]{enumitem}

\usepackage{amsmath}
\usepackage{amssymb}
\usepackage{mathtools}
\usepackage{listings} 
\usepackage{amsthm}

\usepackage{algorithmic}
\usepackage{algorithm}

\hypersetup{
    colorlinks=true,
    linkcolor=[rgb]{0.0,0.0,0.9},
    filecolor=[rgb]{0.0,0.0,0.9},      
    urlcolor=[rgb]{0.0,0.0,0.9},
    citecolor=[rgb]{0.0,0.0,0.9},
    pdftitle={Beyond Statistical Learning: Exact Learning Is Essential for General Intelligence},
    pdfpagemode=FullScreen,
    }

\usepackage{thm-restate}

\usepackage[capitalize,noabbrev]{cleveref}

\theoremstyle{plain}
\newtheorem{theorem}{Theorem}[section]
\newtheorem{proposition}[theorem]{Proposition}

\theoremstyle{definition}

\theoremstyle{remark}
\newtheorem{remark}[theorem]{Remark}

\DeclareMathOperator{\sgn}{sgn}

\newcommand{\cX}{\mathcal X}
\newcommand{\cY}{\mathcal Y}
\newcommand{\cM}{\mathcal M}
\newcommand{\cZ}{\mathcal Z}

\newcommand{\cH}{\mathcal H}
\newcommand{\cA}{\mathcal A}
\newcommand{\cD}{\mathcal D}

\newcommand{\R}{\mathbb R}

\newcommand{\ip}[1]{\langle #1 \rangle}

\newcommand{\norm}[1]{\Vert #1 \Vert}
\newcommand{\zeros}{\bm 0}

\newcommand{\argmin}{\operatornamewithlimits{arg\,min}}
\newcommand{\sind}{\bm 1}

\newcommand{\E}{\mathbb{E}}

\newcommand{\PP}{\mathbb{P}}

\newcommand{\cone}{\operatorname{cone}}

\newcommand{\cFlin}{\mathcal F_{\text{\tiny{\textsc{lin}}}}}

\newcommand{\cF}{\mathcal F}

\newcommand{\leftbin}{\operatorname{\textsc{left}}}
\newcommand{\rightbin}{\operatorname{\textsc{right}}}
\newcommand{\Flin}{\cF_{\textrm{lin}}}

\title{Beyond Statistical Learning: Exact Learning Is Essential for General Intelligence}

\author{%
Andr\'as Gy\"orgy, Tor Lattimore, Nevena Lazi\'c, Csaba Szepesv\'ari\\
 Google DeepMind \\
 \texttt{\{agyorgy,lattimore,nevena,szepi\}@google.com}
}
\date{}

\begin{document}

\maketitle

\begin{abstract}
Sound deductive reasoning---the ability to derive new knowledge from existing facts and rules---is an indisputably desirable aspect of general intelligence.
Despite the major advances 
of AI systems in areas such as math and science,
especially since the introduction of transformer architectures, it is well-documented that even the most advanced frontier
systems regularly and consistently falter on easily-solvable deductive reasoning tasks.
Hence, these systems are unfit to fulfill the dream of achieving artificial general intelligence capable of sound deductive reasoning.  We argue that their unsound behavior
is a consequence of the statistical learning approach powering their development.  To overcome this, we contend that
to achieve reliable deductive reasoning in learning-based AI systems, researchers must fundamentally shift from optimizing for statistical performance against distributions on reasoning problems and algorithmic tasks to embracing the more ambitious exact learning paradigm, which demands correctness on all inputs.
We argue that exact learning is both essential and possible, and that this ambitious objective should guide algorithm design.
\end{abstract}

\section{Introduction}

\emph{Deductive reasoning}, the ability to derive new knowledge from existing facts and rules, is a fundamental aspect of
intelligence.
Deductive reasoning allows a system to derive a vast number of specific conclusions from a small set of general rules,  even if those rules
are completely new to the system.
As such,
systems equipped with deductive reasoning enjoy
an unparalleled form of informational leverage.
Such a system can also discover whether a set of rules is self-contradictory, which is essential for problem solving.
The list of potential practical applications of systems that use natural language to interact with users where deductive reasoning is useful is limitless, 
with examples including everyday problems such as compiling shopping lists for groceries based on plans for the next days, or reliably navigating tax laws, or more specific tasks like calculating insurance premiums and payouts based on policy rules
or determining precise shipping costs and delivery options from logistical tables.

A crucial aspect of deductive reasoning is that all reasoning steps and conclusions have to be correct, so it leaves no room for errors. While such error-free behavior is inherent to classical, rule-based AI systems (provided that they are free of programming errors), error-free behavior is not guaranteed for modern AI systems.
This is because these systems are powered by large language models (LLM) trained to minimize next-token prediction error against some data 
\citep[e.g.,][]{gpt424,team2023gemini,TheC3,grattafiori2024llama3herdmodels}
and the statistical approach guarantees good results 
only on average and only over the distributions close to that of the training data, 
that is, only under ``small covariate shift'' 
\citep{quinonero2022dataset,pmlr-v9-david10a}.

Accordingly, frontier AI systems are well-documented 
to regularly and consistently falter on easily-solvable deductive reasoning tasks.
For example,
\citet{nezhurina2025alicewonderlandsimpletasks} report failures
on elementary math word problems that require 
    skills to add/subtract small integers
    and an understanding of basic family relations;
\citet{mccoy2024embers,mccoy2024language} report failures 
on other 
    simple deductive tasks such as counting, article swapping, shift ciphers, 
    and performing calculations that involve one addition and one multiplication.
Studying LLMs' planning abilities,
\citet{TravelPlanner24} finds that GPT-4 achieves 0.6\% success rate in various travel planning problems.
This agrees with the experience of 
\citet{valmeekam2025systematic} who wrote: ``The lackluster performance of LLMs on even the easiest static test set leads us to continue to believe that
planning cannot be generally and robustly solved by approximate retrieval alone.''%
\footnote{By approximate retrieval they mean using LLMs.}
For other works of this type, consult
\citet{
valmeekam2023planning,
dziri2023faith,
jiang2024peek,
mirzadehgsm,
taghanaki2024mmluproevaluatinghigherorderreasoning,
opedal2024mathgap,
Huckle2025-rm}.
While newer models often fix the problems reported (at least some of them), 
it is unclear whether the improvements are foundational, or they are more localized and we can expect to see new problems discovered and reported. So far, history shows that despite all the advances, people always discover new problems of the same character as the old ones
\citep[e.g.,][]{Huckle2025-rm}, which raises doubts as to whether the advances are foundational and whether continuing this process for developing new models without any changes will result in all the problems fixed.

Experiments with synthetic reasoning (algorithm learning) tasks  corroborate these findings. For example,
\citet{lee2024} and \citet{mcleish2024transformers} train (decoder-only) transformers, the neural networks powering LLMs.
The task is either multidigit addition, or multiplication. 
While these works report major advances, 
despite various simplifications to the task
and task specific interventions,
perfect ``length-generalization'' is never achieved.
While studying their capabilities to learn 
finite-state semi-automata, \citet{liutransformers23}
find that transformers learn ``statistical shortcuts'' to them,
which ``generalize near-perfectly in-distribution [..] \{but\} lack out-of-distribution robustness [..].''
Another work, that is particularly relevant for us, is due to \citet{zhang2023paradox} who demonstrates that 
(bidirectional) transformers 
fail to learn to solve simple reasoning problems
even though they do have the representation capacity.
In particular, the authors of this latter work observe
that for various non-adversarially chosen distributions,
even if the trained models perform perfectly on new test examples drawn from the training distribution, their performance drops to  chance level when tested on an alternative distribution.
Similar results are also reported, among others, by \citet{csordas2021devil,deletangneural23chomsky} and \citet{thomm2024limits}.

Besides documenting failures, much research has aimed at mitigating future failures.
Techniques for doing this include
manipulating training data 
either to reflect the ever-growing list of benchmarks,
or because the changed data is believed (and sometimes proved) to represent an easier learning task
\citep{oren-etal-2020-improving,wang2025learning},
or adjusting the algorithms/losses
 \citep{fuhungry,deletangneural23chomsky, deepseek24,thomm2024limits,Han2025-vj},
or changing the model class (i.e., the architectures of the underlying neural networks)
\citep[e.g.,][]{
bahdanausystematic29,
dehghaniuniversal19,raganato-etal-2020-fixed,you-etal-2020-hard,ontanon-etal-2022-making,fuhungry,zhang2023unveilingtransformerslegosynthetic,Csordas2023-eb}.

The extensive work devoted to documenting and addressing the failures of these AI systems 
on ``off-distribution'' data 
clearly indicates
that developers expect their
systems to be able to generalize to new, unseen data regardless of its origin.
While this is clearly impossible to achieve
in general (e.g., when the solution to a problem is only defined through data, as in image or speech recognition), 
the case of teaching models to perform reasoning does not need to be hopeless:
In reasoning, the rules are clear and, as such, 
for well-formed inputs the set of correct responses is well-defined.
Hence, we posit that when models are taught to reason, expecting them to perform well on new well-specified inputs is reasonable and highly beneficial,
even if in a strict sense generalization to \emph{all} novel scenarios is not necessarily needed \emph{in practice}.
Indeed, on the one hand, flawless deduction is achievable, as formal systems demonstrate.
Achieving this perfection
ensures safety and reliability, 
which has the added benefit of
relieving designers 
from the challenging task of deciding which failures are acceptable--a problem with no perfect answers.
To overcome this problem, we posit that reliable, widely applicable intelligent systems should be built to perform deductive reasoning flawlessly for all distributions over well-formed inputs.
This in turn means that they need to work flawlessly for \emph{all individual} well-formed inputs.
This demand is not new; in fact, we see this as essentially the goal of building AI systems that exhibit \emph{systematic generalization} \citep{fodor1988connectionism,bahdanausystematic29,Csordas2023-eb}.
However, unlike \citet{fodor1988connectionism} whose position is that systematicity is a property of minds and that connectionism (modeling cognitive processes via artificial neural networks) cannot achieve this, we avoid philosophical arguments and instead 
take a pragmatic standpoint stating that some form of systematicity is a property that general AI systems that require deductive reasoning should have and that we need to develop learning algorithms that achieve this.
At the same time, we also propose to replace the somewhat vague notion of systematicity with the more general but crisply defined notion of 
\emph{exact learning} due to \citet{angluin1988queries},%
\footnote{
An even earlier attempt to formalize a form of exact learning emerged in the context of language acquisition \citep{Gold67}, which was later expanded by \citet{Osherson1984-oe} in the context of learning \emph{natural} languages.
}
again, shifting the emphasis to the need of designing learning algorithms of a certain kind.

\emph{Exact learning}, as introduced by 
\citet{angluin1988queries}, studies
the problem of designing methods that learn to apply rules/algorithms correctly on \emph{all} possible inputs.
As noted beforehand, exactly learning is not an explicit goal in statistical learning.%
\footnote{In the terminology of \citet{fodor1988connectionism}, statistical learning is ``empiricism''. Here, we will not argue against all forms of empiricism. On the contrary, we think some form is necessary. However, we argue that
sticking to the purest form of empiricism is likely to slow down progress.
}
As explained above, we argue that the widespread adoption of statistical learning paradigms in AI systems, which optimize for performance against distributions rather than guaranteeing universal correctness, is fundamentally misaligned with the demands of exactness.
We contend that this misalignment is at the heart of the documented unpredictability and limitations of current advanced AI systems. 
Adopting the more stringent exact learning criterion is not merely a theoretical preference; it is a necessary pivot that promises to unlock new research directions, enabling the next significant leap towards closing the gap between current AI systems—which excel at language and lexical knowledge but are unpredictable in tasks requiring logical inference—and the elusive goal of general intelligence.
We thus take the following position:
\begin{quote}
\textbf{
To achieve reliable deductive reasoning in learning-based AI systems, 
researchers must fundamentally shift from optimizing for statistical performance against distributions of reasoning problems and algorithmic tasks to embracing the exact learning criterion, which necessitates universal correctness. }
\end{quote}

The astute reader may ask why the focus is on AI systems that rely on learning. This is because, in order for the system to be useful, it needs to interact with the real world using natural language and  other I/O modalities (e.g., a robot interacting with the world). Since natural language is highly variable, our best chance of achieving this is through learning systems that can capture its semantics, including differences in phrasing. Another advantage of learned natural language systems over formal systems is that formal systems require all knowledge to be explicitly encoded, which quickly becomes very cumbersome and makes such systems brittle \citep{clark2021transformers}. With natural language, we may hope to overcome some of the knowledge gaps using semantics. For example, given that "all buildings have windows", we could conclude that "my house has windows" without explicitly adding that "a house is a building" (admittedly, determining when such implicit rules can be applied is a challenge).

The rest of the paper is organized as follows:
In \cref{sec:notation} we review the foundations of statistical learning and introduce exact learning.
In the next section (\cref{sec:exch}), we elaborate on what goes wrong 
when one attempts to use the statistical approach to achieve exact learning.
The most important message is that good statistical performance does not imply good performance in exact learning, 
and hence we cannot expect statistical learners to do well on exact learning tasks in general.
To illustrate how significant this mismatch is,
we prove that even for simple problems like finding the coefficients of a linear classifier over a binary hypercube,
statistical learners will pick up ``statistical shortcuts'' and fail to acquire the target rule, unless 
the learner and the distribution is carefully tailored to the task
or an exorbitant number of training examples are used. We give a new lower bound that characterizes the critical number of examples needed by a statistical learner to achieve nontrivial performance on exact learning. 
We next use this to give lower bounds on this critical sample size
for symmetric learners that is \emph{specific for both the 
target hypothesis and the training distribution}. Importantly, these results apply to neural networks like  transformers when trained with gradient methods. 
In \cref{sec:prom} we describe a number of potential approaches 
to address the previously mentioned challenges.
\cref{sec:alternative} discusses some plausible arguments against pursuing exact learning, complete with brief  counter-arguments.
Finally, conclusions are drawn in \cref{sec:conc}.

\section{From statistical to exact learning}
\label{sec:notation}

At an intuitive level, the principle behind statistical learning
is that if model performance is evaluated on average with respect to some distribution, then 
designing models that perform well on a set of training data sampled from \emph{the same distribution} will lead to models with good performance. 
We will now describe the formal mathematical model behind statistical learning as this will facilitate our discussion of the similarities and differences between statistical and exact learning.

Learning can be formulated as the goal of keeping losses small.
To describe this let $\ell(z, \theta)$ be the loss when a model with parameters $\theta \in \Theta \subset \R^d$ is faced with example $z$ coming from an example space $\cZ$. Generally speaking $\Theta$ is a set of possible model parameters (e.g., weights of a neural network) and $\cZ = \cX \times \cY$ where $\cX$ is some space of inputs (e.g., all images, all character sequences) and $\cY$ is a space of labels (labels are interpreted here as anything we want to predict, including the class of an image, the value of the next token in text prediction, or a whole token sequence,
or the reward in a reinforcement learning setting).
The learning system is assumed to have access to examples $Z_1,\ldots,Z_n$ sampled (usually) independently from an unknown distribution $\mu$ on $\cZ$. In statistical learning the objective is to find parameters $\theta \in \Theta$ minimizing the expected loss
\begin{align}
\label{eq:loss}
    L_\mu(\theta) = \E[\ell(Z_1, \theta)]\,.
\end{align}
The fundamental principle in statistical learning is that
in great generality it suffices to approximately minimize the empirical loss, which is
$L_n(\theta) = \frac1n \sum_{i=1}^n \ell(Z_i,\theta)$.
At least for well-behaved losses, the concentration of measure phenomenon shows that with overwhelming probability 
$L_\mu(\theta)-L_n(\theta)$
 is of order $O(1/\sqrt{n})$ or less for any prespecified $\theta$ \citep{boucheron2003concentration}. 
 It is not a great leap from here to argue that learning algorithms approximately minimizing $L_n(\theta)$ also approximately minimize $L_\mu(\theta)$. 
 Methods for making this argument rigorous include uniform concentration \citep{shalev2014understanding,dudley2014uniform}, algorithm-dependent methods like stability \citep[Chap 7]{zhang2023mathematical} and PAC-Bayes \citep{alquier2024user}.

The statistical learning principle works well when the data arising in deployment has the same distribution as the training data. This is rarely how things work in practice.
In reality, and in particular in the training of AI systems based on LLMs, system designers try hard to collect data that appears to have ``good coverage''. 
Then they minimize the empirical loss and subsequently spend considerable effort testing the system on an ever-expanding set of benchmarks; for example, the reports accompanying the release of each new LLM always contain performance report on a large set of benchmarks \citep[e.g.,][]{gpt424,team2023gemini,TheC3,grattafiori2024llama3herdmodels,team2025gemma}.\footnote{An alternative popular approach to test the quality of LLMs is to let users interact with multiple systems (without knowing their identity) and then indicate their preferences \citep{chiang2024chatbot}. The models are then ranked based on Elo scores. The results reflect the preferences of the population that self-selects to visit the  website.}
When new benchmarks are identified where a model behaves poorly, more data is added, the model is retrained and the whole process is repeated. The explanation for the failure is that the model cannot be expected to generalize to parts of the input space not ``covered'' by the training distribution, and also because the loss functions used in training and test might be significantly different.\footnote{This is further complicated by the multi-step nature of the training process, containing pre- and post-training, as well as supervised fine-tuning, which may all come with their own datasets and loss functions.}

We claim that the above iterative process of training models and adding new training data which tries to fill the gaps in the test performance is a symptom of the mismatch between the statistical approach to machine learning and the 
 true underlying hidden goals of the developers of these systems, especially those that have a general purpose. 
 Indeed, having a general purpose system means that even if there \emph{were} some distribution underlying the future inputs to be used with the system, this distribution $\rho$ is not available to the developer. As such, the developer is forced to consider \emph{all possible  input distributions}, while keeping the conditional distribution of the labels, $\mu_{Y|X}$, fixed.
The goal of such a developer is what we identify as the \emph{goal of exact learning}, which is to minimize
 \[
  L^*_{\mu_{Y|X}}(\theta)= \sup_{x\in \cX} \int \ell( (x,y),\theta) \mu_{Y|X}(dy|x)\,.
 \]
 The integral here is only over the randomness in the labels (that is, it is a conditional expectation over the labels given the input $x$). The supremum over all $x$-values means that $L^*_{\mu_{Y|X}}(\theta)$ can only be minimized when the model determined by $\theta$ generalizes over the entire range of $x$-values.
 For example, in a deductive natural language reasoning task,
 an example $z=(x,y)$ would be so that $x$ is the a well-formed reasoning task specification, $y$ is a possible response,
 and $\mu_{Y|X}(\cdot|x)$ is a distribution supported on the set of correct answers. Then choosing $\ell( (x,y),\theta )$ to be the indicator of whether the answer produced by the model with weights $\theta$ is semantically different from $y$, we get that a learner
 achieving zero loss produces correct answers on any reasoning task in $\cX$.
 As compared to the standard definition of exact learning
 which uses fixed labeling functions (i.e., $\mu_{Y|X}(\cdot|x)$ is a Dirac),
 ours is more general and aligns better with ``messy'' real-world tasks. Overall, we do not think this generality gives rise to essential differences. 
 Another difference is that we allow other losses than the zero-one loss. While this is also not the focus and soon we will switch to the zero-one loss, using the definition with say, real-valued labels and some loss that fits this setting (i.e., the squared loss) allows one talk about exact learning in regression settings. In this case, the goal can be seen as controlling worst-case prediction errors, which is a goal often needed in numerical analysis or reinforcement learning and has been studied in approximation theory and reinforcement learning
 \citep{lorentz1966approximation,du2020good,lattimore2020learning}.
 
 With the above definition in hand, we
 can now interpret the meaning of a learner picking up a  \emph{statistical shortcut}, which is a phrase often used in the literature \citep[e.g.,][]{geirhos2020shortcut,liutransformers23,zhang2023paradox,taghanaki2024mmluproevaluatinghigherorderreasoning}.
 In what follows we will say that a learner learned a \emph{statistical shortcut} when it has a ``large''   loss relative to the best possible loss. Here, the meaning of large needs to be adjusted to the context.
 When working with the zero-one loss we choose large to mean the loss that one achieves when picking labels uniformly at random, independently of the input. 

It is not hard to see that $\sup_{\rho} L_{\rho\otimes \mu_{Y|X}}(\theta) =L^*_{\mu_{Y|X}}(\theta)$; that is, exact learning is identical to the worst-case out-of-distribution generalization. 
Hence, a statistical learner that does well against \emph{all possible} distributions is a good exact learner and vice versa.
However, there is no guarantee that a statistical learner that will do well in the statistical sense will also do well on exact learning (as we also shall see in the next section).
This also extends to the case of learners that do well against a few select distributions. This is why the standard practice of developing algorithms that do well on a fixed set of benchmarks does not guarantee good performance on a ``new'' benchmark.

Learning to perform well on a well-defined set of deductive reasoning tasks described in some formal language is a problem of \emph{learning an algorithm}.
Examples of algorithm learning problems include
learning to perform a subset of all arithmetic operations, 
learning to solve a system of formally described linear equations,
learning to solve arithmetic word problems,
learning to follow instructions/some rules,
learning to decide whether a proof written on some formal language is correct, to name a few. 
When it comes to evaluating the success of  an algorithm, one may consider the average rate of correctness---this would correspond to using a statistical approach and requires a distribution that reflects the use of the algorithm when it will be deployed. Since algorithms are meant to be general, such distributions are not (necessarily) available, in which case it is natural to require that the algorithm is \emph{always correct}, leading to the exact learning criterion.

\paragraph{Binary classification with linear classifiers}
The most celebrated result of statistical learning is that learning is largely insensitive to the distribution of inputs.
Intuitively, the reason for this is because the training set examples come from the distribution used for assessing performance. Hence, even if some examples are rare and thus will not likely to appear in a fixed size training set, downstream performance will not be effected by this, since they are also not likely to appear at ``test time''. 
In what follows, we explain the details of this result in the context of binary classification with linear classifiers, a simple setting 
that allows us to illustrate how large the gap between exact learning and statistical learning can be.

In linear binary classification
the input set $\cX$ is a subset of a Euclidean space $\R^d$ and the label set is binary: $\cY = \{0,1\}$.
For $\theta = (w,b)\in \R^d \times \R\equiv \R^{d+1}$,
a linear classifier $f_\theta$ over $\cX$ is a binary-valued map that returns $1$ for any point $x$ in the halfspace $\{x\,:\, \ip{w, x} + b \geq 0\}\subset \R^d$. We let $\Flin(\cX)$ to denote the set of linear classifiers over domain $\cX$.
The loss function $\ell$ is also binary-valued and at $((x,y),\theta)\in \cZ \times \R^{d+1}$ returns one if $f_\theta(x)\ne y$ and zero otherwise. 
Assume now that the distribution $\mu$ of the examples is such that for any $(x,y)\sim \mu$, $y = f_{\theta^*}(x)$ for some $\theta^*\in \R^{d+1}$ that is unknown to the learner. 
The expected loss \eqref{eq:loss} of a classifier with 
parameter $\theta\in \R^{d+1}$ is the probability of error on a random input that has the same distribution as the inputs in the training data.

Now, if one cares about finding a classifier with a small probability of error on  a new random input, as it turns out, it is sufficient to find a parameter vector $\theta_n\in \R^{d+1}$ that makes the empirical loss $L_n(\cdot)$ the smallest possible. Since we assumed the labels are generated by a linear classifier, it is always possible to find a weight vector that makes the empirical loss zero: $L_n(\theta_n)=0$. The so-called VC dimension
of the set $\Flin(\cX)$ is known to be $O(d)$ and thus
classical learning theory \citep[Thm 4.17, Prop 4.18]{zhang2023mathematical}
shows that for any $n \geq 1$ and $\delta \in (0,1)$, with probability at least $1 - \delta$,
\begin{align}
L_\mu(f_{\theta_n}) = O\left( d \log(n/\delta)/n \right)\,.
\label{eq:vc}
\end{align}
Note how demanding a higher success rate (lower $\delta$) has a relatively small cost.
Moreover, the upper bound decreases rapidly as $n$ increases, and only $\Omega(d \log(d/\delta))$ examples are need before the expected \emph{on-distribution}
performance becomes better than choosing labels uniformly at random.
All this holds regardless of the size of $\cX$.
In particular, there is no requirement to see all elements of $\cX$ during training.
Furthermore, all these conclusions hold regardless of the choice of the input distribution and the classifier generating the labels.%
\footnote{The astute reader may wonder about the relevance
of this and similar uniform-convergence-based results 
in the age when neural networks are all the rage. The reader is right to question this, though we note one still expects that increasing the sample size will lead to a polynomially improving performance, as attested by the many works on ``scaling laws'' for neural networks \citep{Bahri24scalinglaws}.
}
By contrast, in exact learning, a learner with the same information cannot be expected to perform well without seeing \emph{almost all possible inputs}.
Intuitively, this is because the definition requires the learner to \emph{generalize out-of-distribution}
in the most extreme sense. In particular, $L_{\mu_{Y|X}}^*(\theta)  = \sind(f_\theta\neq f_{\theta^*})$.

\section{Naive statistical approaches to exact learning fail}
\label{sec:exch}
We now explore the failure of classical statistical learning recipes
to learn exactly.
In the most naive approach,
a developer would collect a dataset from some input distribution. The ideal statistical model for this is to assume that the data is independently and identically distributed, as in the previous section. 
We start with an example that shows that a distribution that has full support over the whole space is far from sufficient in this case:
\begin{restatable}{proposition}{silly}
\label{prop:silly}
Let $\cX= \{0,1\}^d$, and consider the two linear classifiers $f^*_1(x) = \sind(x = \zeros)$ and $f^*_0(x)\equiv 0$.
Let $n = o(2^{d})$ and let $X_1,\dots,X_n$ be drawn independently from uniform distribution over $\cX$. Then any learner fails to identify the target function $f^*$ with constant probability either when $f^*=f^*_0$ or when $f^*=f^*_1$.
\end{restatable}

This result illustrates the separation between the hardness of exact learning and statistical learning. For example, from \cref{eq:vc} it follows that $\tilde O(d^2)$ samples are sufficient to bring the average loss below $1/d$ with high probability, while the proposition shows that
any learner that exactly identifies the correct labeling hypothesis with probability at least $1/2$ needs to see $\Omega(2^d)$ examples,
which is essentially all of the possible inputs.
While here the input space is finite, even for moderately large $d$ (such as $d=100$), the number of possible inputs is an astronomical number. In connection to this, we note that
we choose the domain $\{0,1\}^d$ as in the applications we can think of, the input space has a similar combinatorial structure and is subject to the same explosion as what we see here.%
\footnote{
By applying problem transformations, as we shall see it later, it is perhaps possible to keep the domain finite, though we are not hopeful that these reductions would allow to reduce the size of the domain to a \emph{small} finite number as the domains one ends up with are still combinatorially structured.
}
The example also demonstrates that \emph{coverage of the input distribution} in terms of having a \emph{full support over the input domain} is far from sufficient for exact learning. 
None of this is surprising. After all, under the input distribution for two distinct labeling functions, the learner only sees one type of label until they see almost all the points in the domain. 
Yet, this example should remind everyone that 
lack of coverage of the input distribution (which is almost always what is blamed when learning does not ``succeed'') is only one of the many  reasons that a learner will fail to generalize to unseen examples.
Let $\cH$ be a set of maps from $\cX$ to $\{0, 1\}$.
Proposition~\ref{prop:silly} is a corollary of the following
theorem, which shows
that for \emph{any} input distribution $\rho$ on $\cX$ there is a lower bound on the number of samples required for successful exact learning by \emph{any} learner provided that the learner faces the challenge of identifying an arbitrary classifier from $\cH$.

\begin{restatable}[All learners fail on some task for small samples]{theorem}{alllearnersfail}
\label{thm:dis}
Fix a distribution $\rho$ over $\cX$ and let $X\sim \rho$.
With a sample from $\rho$ that has fewer examples than
$1/(2\inf_{h,h'\in \cH, h\ne h'} \PP(h(X)\ne h'(X)))$,
 all learners fail to identify 
the function used to label the sample for at least one of the labeling functions $h^*\in \cH$
with probability of at least $1/4$. 
\end{restatable}
The proof of this and all remaining results are given in the appendix.
The idea is similar to that underlying Proposition~\ref{prop:silly}. If for two functions $h\ne h'$, no example is seen from the \emph{disagreement region} $\{x \in \cX:h(x)\ne h'(x)\}$, the learner has no way of distinguishing between whether the sample was labeled with $h$ or $h'$. 
The theorem gives a lower bound on the minimum sample size needed for exact learning that depends on the interplay between the input distribution and the hypothesis class $\cH$. Note that there is an elementary upper bound that matches the lower bound up to an $O(\log |\cH|)$ factor \citep{shalev2014understanding}.

The example in \cref{prop:silly} is perhaps unsatisfactory as the slow learning happens because one of the labels is not seen for a long time. 
Equipped with the above result, we now give an example where this is not the case, and the input distribution $\rho$ is the uniform distribution which is (in many senses) the best uninformed choice in terms of coverage of the input space, which we keep as $\cX = \{0,1\}^d$. Now take any $h,h': \cX \to \{0,1\}$ functions that return $0$ on about a constant proportion of the domain elements and differ on only $o(2^d)$ elements.
The above theorem immediately shows that no learner is able to reliably solve the exact learning problem in polynomial time in $d$ where the labeling function could be either $h$ or $h'$. 
This problem can be instantiated with linear labeling functions, as well: 
Assume for simplicity that $d$ is even and let $m=d/2$. Treat the $d$-component binary vector as the binary digits of two $m$-bit numbers. 
Let the two numbers (as integers) for $x\in \{0,1\}^d$ be $\leftbin(x)$ and $\rightbin(x)$.
Let $h_{\ge}(x) =\sind(\leftbin(x)\ge \rightbin(x))$ and $h_{>}(x) = \sind(\leftbin(x)>\rightbin(x))$.
It is not hard to see that 
{\em (i)} both $h_{\ge}$ and $h_{>}$ are linear classifiers
and when $X$ has law $\rho$,
{\em (ii)} $\PP( h_{\ge}(X)\ne h_{>}(X) ) = \PP( \leftbin(X) = \rightbin(X) ) = 2^{-m} = 2^{-d/2}$;
and
{\em (iii)} $\PP( h_{\ge}(X)=0 ) \approx \PP( h_{>}(X)=0 ) \approx 1/2$.
Summarizing: $\rho$ has full support, the hypotheses are simple (linear) with balanced  label distributions, and no learner can tell $h_{\ge}$ apart from $h_{>}$ unless they see $2^{d/2}$ samples or more. Hence, even simple learning problems with seemingly ideal conditions can be very hard if the goal is exact identification.
One does not need neural networks to find that learners will learn 
\emph{statistical shortcuts}. Moreoever, on some problems, 
the shortcuts are unavoidable and thus
it is not the fault of a particular learner that they learn a shortcut!

\paragraph{Task-specific lower bounds: Symmetric learners pay a price}
A weakness of the previous result is that it says nothing about how long a specific learner will take on a specific learning problem, which is arguably a question of more relevance to practice. After all, we only need one learner to succeed on the single task of becoming good at natural language reasoning.
As it turns out, the previous result allows us to derive a similar lower bound for any specific \emph{symmetric} learner on any given learning problem.

Most learners used in practice have various symmetries; formally, this means that they are equivariant to various transformations of their inputs.
One such property is label symmetry: 
A learner $\cA$, which maps training data to predictors, is said to have \emph{label symmetry} if 
for any $(x_1,y_1),\dots,(x_n,y_n)\in \cZ$ it holds that
$\cA( (x_1,1-y_1),\dots,(x_n,1-y_n) ) = 1- \cA( (x_1,y_1),\dots,(x_n,y_n) )$.
That is, flipping the labels in the training data causes the learner to flip their predictions.
Another such property is variable symmetry.
A learner $\cA$ is said to have \emph{variable symmetry} if reordering the components of the input vectors in the data that the learner sees 
makes the learner return a function 
	that returns the same values on some input $x$ 
	as the function it learns on the original input 
		when that function is applied to $x$ but with the component's of $x$
			reordered the same way as the components were in the input vector.			
When a learner randomizes (i.e., because it uses random initialization, or picks training examples at random), we extend these concepts accordingly to apply to the probability kernels that defines them in the natural way.			
With this, we can see that multilayer perceptrons trained with various gradient methods have variable symmetry (see e.g. \citet{ng2004feature}). 
They also have label symmetry when the last layer weights are initialized from a distribution that is centrally symmetric. 
Transformers trained with gradient descent do not have variable symmetry, but if they are trained on a binary classification problem, they would have label symmetry.
Moreover, they also have input alphabet symmetry (or token space symmetry): They do not treat any of the input tokens as special.
More generally, given a group $G$ of symmetries, an algorithm is said to be $G$-symmetric if
transforming the data with some element $g$ of $G$ and then applying the learning algorithm gives 
the same outcome as applying the learning algorithm on the untransformed data and then taking the (appropriately defined) action of $g$ on the resulting function: $\cA( g \cD ) = g \cA( \cD)$.
With this, one readily derives the following result from \cref{thm:dis}:
\begin{restatable}[Distribution and hypothesis specific critical sample size for symmetric learners]{theorem}{symlearnersfail}
\label{thm:dis2}
Let $\cA$ be a $G$-symmetric learner, $\rho$ a $G$-symmetric distribution,  $X\sim \rho$. Let $h: \cX \to \{0,1\}$.
With a sample from $\rho$ that has fewer examples than
$1/(2\inf_{g\in G \setminus \{\mathbf{1}_G\}} \PP(h(X)\ne gh\,(X)))$,
$\cA$ will fail to return $h$
with probability of at least $1/4$
when used on a sample where the inputs are generated from $\rho$ and the labels are from $h$.
\end{restatable}

As said before, multilayer perceptrons trained with gradient descent posses variable and label symmetry.
As such, one gets that \emph{these methods will fail on learning the specific hypothesis $h_{\ge}$ unless the sample size is at least $2^{d/2}$}.
A similar result applies to transformers and of course, $h_{\ge}$ is just one of the many hypothesis which are hard to learn with learners with symmetries like the one mentioned.
Thus, we see that the symmetry of a learner can be used to argue that the learner will take some time before finding a specific hypothesis.
This should make intuitive sense:
A symmetric learner is one that lacks certain biases (e.g., no preference for label zero over label one).
As such, they need to consider various alternatives for which there is a price to pay. 
Conversely, a non-symmetric learner, such as the one that returns $h_{\ge}$ regardless of the training data, has no such limits. Obviously, this particular non-symmetric learner is not desirable unless the goal was to learn $h_{\ge}$.

Our argument that learner symmetry slows exact learning can be seen as a precise, mathematical version of 
one of \citet{fodor1988connectionism}'s main argument, where they
posited that 'generalist' connectionist methods based on 'empiricism' are bound to fail at systematic generalization.
There are however two important differences: Firstly, our argument applies to any symmetric learner, not just connectionist methods. 
Secondly, as we will discuss later, methods that are specialized to exact learning can avoid slow learning even if neural networks are involved.

At this stage it maybe useful to recall that in learning to reason with natural language input, 
we use learning methods because the map from inputs to the desired outputs is not available 
and the best we can do is to evaluate the correctness of specific responses to specific prompts. 
Then the flexibility of natural language forces one to use learning methods that may need to have a lot of symmetries, and as we have seen here, the more symmetries the learner has the harder exact learning will become.
Unsurprisingly, symmetries of learners have been used in the past to construct similar lower bounds in the statistical setting \citep[e.g.,][]{Malach2021-jo,AbbeBo2022sym,Abbe2022-xg}
and a very early intuitive formulation of these ideas can be traced back to at least the work of \citet{mitchell80}. We will add more on the role of symmetries in exact learning in the next section.

\paragraph{Gradient descent can be slow}
The next hurdle we mention here is that 
 convergence (if it happens at all) can be very slow
when neural networks are trained with gradient descent and cross-entropy loss.
This is relevant because this is exactly the procedure that powers modern LLMs (neglecting the specific choice of the gradient-based optimizer used).
This slow convergence is best examined in the special case of training linear classifiers.
In this context, \citet{soudry2018implicit} showed that for linearly separable data the weights converge to the \emph{maximum-margin} linear classifier. This is the classifier that has the property that the hyperplane which defines the boundary between the positive and negative classes is the furthest 
away in the standard Euclidean distance from both the positive and the negative training examples. 
The upper bounds provided by
\citet{soudry2018implicit} suggest the convergence might be rather slow, which is confirmed experimentally (\cref{fig:lin:gd1}), where, to avoid issues with learning rate tuning, we simulated gradient flow. The figure shows for different values of $m$ the evolution of the training loss when learning $h_{\ge}$. Note the exponential scaling of the $x$ axis. The dots indicate when a weight is reached where exact learning is achieved.
\begin{wrapfigure}[28]{r}{0.34\textwidth}
  \centering
  \vspace{0.1em}
  \includegraphics[height=4.4cm,clip]{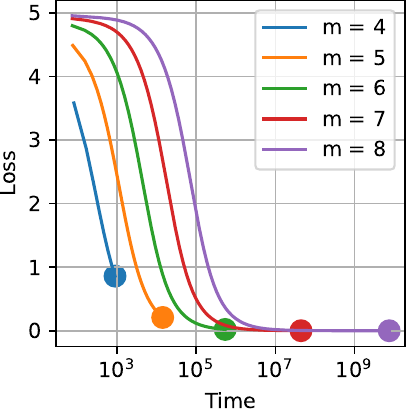}\vspace{-0.3em}
  \caption{Gradient descent is slow to learn exactly on a good teaching set.}\label{fig:lin:gd1}
  \vspace{1.2em}
    \includegraphics[height=4.4cm,clip]{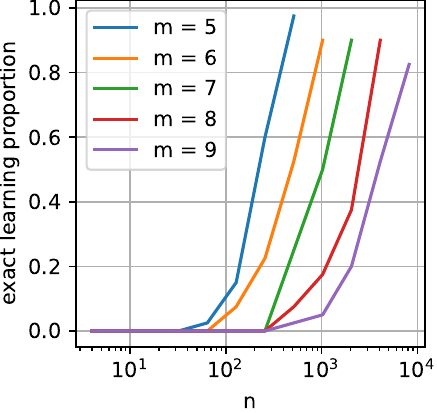}\vspace{-0.3em}
  \caption{A maximum-margin classifier can require many examples.}
  \label{fig:lin:gd2}
\end{wrapfigure}
To make exact learning feasible, we use a small dataset of $O(m)$ examples for which exact learning is guaranteed to be achieved by the maximum-margin classifier (see the next section for details on this). A possible solution is to use ``aggressive'' stepsizes, as suggested by \citet{nacson2019convergence}. It remains to be seen whether such stepsizes can be used in more complex scenarios such as with LLMs.

\paragraph{Surrogate losses may delay or prevent exact learning}
The standard practice to train transformers is to use the cross-entropy loss. For some tasks using this loss may actually  prevent convergence to the exact solution.
As a striking example, \citet{vandenbroeck2023} shows that if the weights of a transformer are initialized to exactly solve a particular logical reasoning problem, training the transformer on instances of the problem using next-token prediction can destroy the exactness of the solution. To understand this, consider the analogous problem of learning $h_{\ge}$ using cross-entropy loss. Gradient descent with this loss eventually converges to the maximum-margin solution, which for datasets that are not either very large or carefully chosen teaching sets does not yield exact learning. \cref{fig:lin:gd2} shows the evolution of the proportion of inputs with correct responses for different values of $m$ and the maximum-margin classifier. Even in this simple problem, the number of examples needed before exact learning is achieved grows exponentially with $m$. The examples here are sampled from a distribution that is uniform on the support vectors of which there are exponentially many. We expect similar phenomenon to appear in the more complex natural reasoning tasks. Note that in this example we see that the bias of a learning method (coming from gradient descent and the cross-entropy loss) negatively interferes with the success of exact learning. Had
 \citet{vandenbroeck2023} used the $0/1$ loss for training, or even possibly the hinge-loss, unlearning would not have happened. Of course the $0/1$ loss is notoriously hard to optimize against and the hinge loss is not a good fit for the multiclass/multilabel setting of next-token prediction tasks.

\section{Beyond the naive approach}
\label{sec:prom}

Next we discuss possible methods for overcoming the limitations of the naive statistical learning approach considered in the previous section.
We focus on directions that are best suited to the specifics of using learning techniques to produce an AI system that needs to use deductive reasoning to ensure its correct operation. What is special about this problem is that we know very well what valid inference looks like. While the inputs and outputs are outside of the realm of formal systems, when an input is sufficiently specific, there is no debate about what the correct solution is.%
\footnote{Here, we would like to acknowledge that problems specified using natural language rarely if ever are fully specific, as it is lucidly illustrated in the little satirical book ``Mathematics Made Difficult'' by \citet{linderholm1971mathematics}, but we think this is not a fundamental issue.
} 
This opens up a number of possibilities. 
While almost all of these are already being explored in the literature, 
new ideas may come from being aware of the difference in demands in exact versus statistical learning. 

\paragraph{Performance evaluation}
In this area, the limited utility of static benchmarks,
or the endless, ad-hoc tinkering of benchmarks should be acknowledged.
In addition, there should be more focus on a systematic understanding of the failures and, more generally, failure modes of current systems. 
Can we develop frameworks where the developed systems are actively challenged?
Or algorithms for verifying correctness, even if the verification is limited in some ways?
One way to start is to draw inspiration from the neighboring literature on adversarial testing 
\citep[e.g.,][]{WoKo18,UKSz19,xu2020adversarial,Leibo21,zolfagharian2023search}, 
formal verification \citep[e.g.,][]{WoKo18,chevalier2024achieving},
mechanistic interpretability \citep{sharkey2025openproblemsmechanisticinterpretability},
or systematic generalization where the standard practice is to use distinct distributions for training and testing; the so-called generalization split \citep{Csordas2023-eb}.

\paragraph{Changing learners}
As it is clear from our previous result, aiming for increased generality (introducing more symmetries to the learner) will delay exact learning. One approach, which was mentioned before, is to remove some symmetries of the learner, by forcing the predictor to be symmetric (equivariant to certain input transformations of the designer's choice), which, at least in the statistical setting, has been shown to reduce sample complexity
\citep{bietti2021sample,tahmasebi2023exact}. 
There are numerous works that explicitly follow this approach
 \citep[e.g.,][]{cohen2016group, zaheer2017deep,scarselli2008graph, battaglia2018relational,bronstein2021geometric}.

\paragraph{Teaching learners}
Since we know what algorithm/hypothesis the learner needs to learn, the learner can be helped by feeding it with data that will make it learn as efficiently as possible.
As an illustration, consider the problem of training a maximum-margin classifier over the binary hypercube $\{0,1\}^d$. As we have seen earlier, this problem may be very challenging when the training data is not carefully selected. 
However, as it turns out, one can prove that 
\emph{no matter the target linear classifier, there always exist a dataset of size at most $2d+2$ such that the maximum margin classifier matches the target} (Proposition~\ref{prop:supp} in the appendix). Intuitively, one needs a few points on the margin that uniquely determine the hyperplane the classifier needs to use.
The dependence of the teaching set on the target function \emph{and the learner} is absolutely essential.
Using the teaching set just mentioned with a multi-layer perceptron instead of the maximum margin classifier is expected to lead to very poor performance. With
such a small amount of data, due to its symmetries, the neural network is not likely to find a solution. Sadly, the more the powerful the learner is, the larger the minimal teaching set will be and calculating a minimal teaching set for a powerful learner may not be feasible.

``Exact learning'' is a term borrowed from the early literature of machine learning. Here, the presence of a teacher is implicit in the model. As introduced by \citet{angluin1988queries}, in this model 
the learner wants to identify a hypothesis from a given class
based on either asking for a counterexample to a hypothesis they have  (``equivalence query''), 
or by asking whether a given point belongs to the target hypothesis (``membership query''). In both cases, we can think of the oracle that gives the answer as the teacher. In the original model, the teacher is doing minimal work.
Indeed, \citet{angluin1987learning} refers to the oracle as the minimally adequate teacher.
However, the problem of designing a good teacher has also been investigated \citep{goldman1993complexity,angluin01,liu2016teaching,Sloan2010-xh,ohannessian2016optimal,liu2017iterative,bharti2024on}.
The main difference is that this literature focuses on the problem of teaching/learning a whole class of hypotheses.
This is unlike our problem where we have a single (but complicated) hypothesis/algorithm to teach the learner, who, however, has to learn other things (because, say, the problems are specified using natural language). Thus, the learner has to be quite general (to be able to deal with the complexities of language), and only part of the training data can be controlled (since the learner also has to learn to understand and use language).

\paragraph{Active learners and active teachers; cooperation}
In the classical learning literature, the learner is directing the learning process \citep{angluin1987learning, angluin1988queries, littlestone1988learning}
In the earliest works, the teacher and the learner did not cooperate; they were assumed to do the minimal amount of work necessary.
The appeal is generality, but this approach can be needlessly wasteful.
\citet{zilles2011models} considers the problem where the learner and teacher can cooperate, though the focus is still on  teaching/learning hypothesis classes.
On the empirical side, much thought went into the related problem of automatic curriculum design \citep[e.g.,][]{dennis2020emergent,portelas2021automatic,parker2022evolving}.
The algorithms in these papers can be seen as attempts to solve a (somewhat generalized) exact learning problem. 
Indeed, these works often refer to the need to achieve compositional/systematic generalization
\citep{kirk2023survey,parker2022evolving}.
One may use learning, sampling and access to a series of models produced during training \citep[e.g.,][]{UKSz19},
or modularity of a hybrid system
\citep[e.g.,][]{dreossi2017compositional}
to find challenging input instances, which can then be fed to training.

\paragraph{Changing the task} 
\begin{figure}[t]
    \centering
    \includegraphics[width=0.8\textwidth, trim={0 4 0 4},clip]{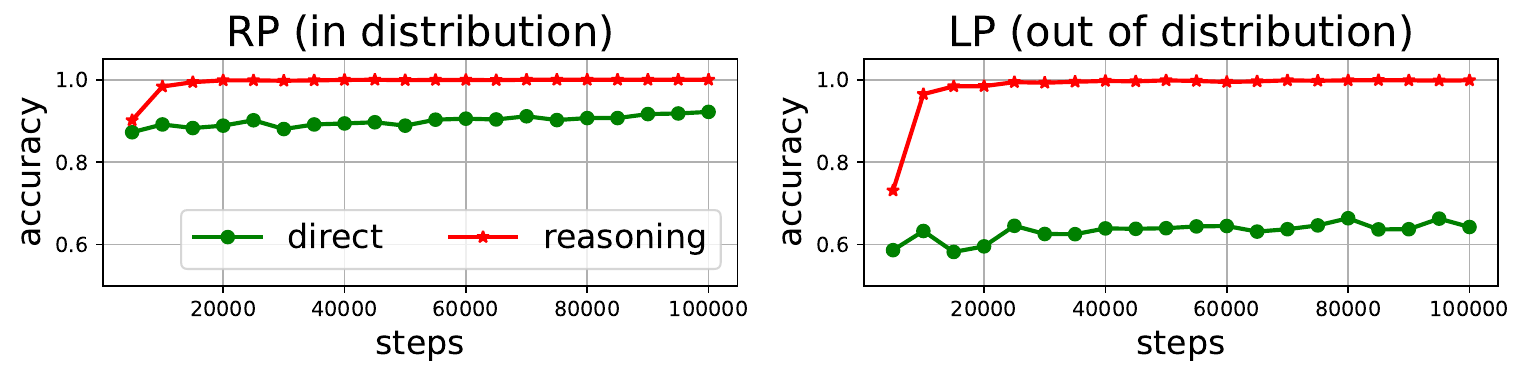}
    \caption{Accuracy on propositional logic problems when training in "direct mode" vs. with reasoning traces. Training on RP data with reasoning traces achieves very high accuracy, but is not exact (1.0 on RP data and 0.999 on LP data.
    }
    \label{fig:reasoning_traces}
\end{figure}
In the more typical passive learning setting with a fixed learner (typically based on gradient descent) and a dataset of labeled examples drawn from a fixed distribution, we may consider changing the task in various ways. 
In some cases, simply changing the loss of the learner may help, as discussed in the section on surrogate losses. 
In the context of learning to reason using next-token prediction, it is a common practice to leave out from the loss computation the tokens that are part of the input.
Contrary to expectations, illustrating the complexity of training LLMs, this has been documented to harm performance
\citep{shi2024instruction,huerta2024instruction}.

Another potential approach involves transforming or augmenting the task and the associated data. 
In the context of LLMs, a particularly promising direction involves learning with chain-of-thought reasoning traces 
\citep{kim2023cot,hou2024universal,li2025llms}.
A gain is expected  if learning the elementary reasoning steps is a lower-complexity learning problem. 
The complexity can be further reduced 
by building a hybrid system that explicitly searches through possible recursive problem decompositions, leaving only the job of suggesting specific problem decompositions and transformations to the learnt system components \citep{yao2023tree,zhang2024tree,zhou2024language}.
By rejecting incorrect solutions, correctness of a system can be ensured by training a verifier, that is, a binary classifier that decides whether a purported solution is correct
\citep{li-etal-2023-making,lightman2023let,kamoi2025training}. 
Training such a verifier may be easier than solving the problem in one-step, especially, if other system components are responsible for breaking up the original problem into small manageable components.

Another example of changing the task (and constructing a teaching set) is given by the work 
of \citet{wang2025learning} who,
in the context of learning to compose functions, 
proved and also empirically demonstrated 
the benefit of training with a ``good mixture'' of problems where the number of functions to be composed was 
appropriately controlled.
Neuro-symbolic approaches attempt to combine the power of symbolic reasoning and learning and many of the approaches tried there can be seen as changing the task
\citep{Garcez2019-xf,Kautz22,hitzler2022neuro,Shakarian2023-pv,garcez2023neurosymbolic}.

To illustrate the power of changing the task,
we consider learning simple logical reasoning with and without reasoning traces.  Here, given some simple rules (definitive clauses) and a set of facts (true zeroth order predicates) the value of a query predicate has to be determined; see Figure~\ref{fig:simple_logic} for an example problem. 
We follow the work of \citet{zhang2023paradox}, which describes two methods for generating such problems, named rule-priority (RP) and label-priority (LP), both of which cover the space and have a similar label distribution. \citet{zhang2023paradox} train a BERT model from scratch on such data (without reasoning) and show that training on RP to near-perfect accuracy does not generalize to LP and vice versa.  We train a decoder-only model on RP problems with and without reasoning traces (see Appendix~\ref{app:experiments} for details), and evaluate on RP and LP. We find that both in-distribution and out-of-distribution performance is much better (near-perfect) with reasoning, though exact learning is still not achieved (see Figure~\ref{fig:reasoning_traces}).  
The key to the improved results is changing the task so that the transformer learns to perform steps of forward chaining rather than implementing multiple steps in-weights.
Teaching the algorithm rather than just presenting a model with input-output pairs effectively changes what statistical shortcuts (if any) the model learns. 

\begin{figure}[t!]
\begin{framed} 
{\color{blue} {\bf Rules:}
 If wrong and frail then impartial .
 If frail then grumpy . 
 {\bf Facts:} frail .
 {\bf Query:} grumpy ? } \newline
{\color{magenta} {\bf Reasoning:}
 Facts: frail .
 If frail then grumpy . 
 Newfact: grumpy .
 Facts: frail , grumpy . } \newline
 {\color{teal} {\bf Answer:} yes}
 \end{framed}
    \caption{An example propositional logic problem with reasoning. We are given a set of rules (definite clauses), a set of facts (true predicates), and a query predicate. The goal is to compute the truth value of the query. The reasoning trace (in magenta) executes the forward-chaining algorithm.}
    \label{fig:simple_logic}
\end{figure}

\section{Alternative views}
\label{sec:alternative}

\paragraph{Existing statistical methods are working} The dominant approach in the field has involved training ever-larger models by hill-climbing on an ever-increasing set of benchmarks using supervised and reinforcement-learning methods. This has led to advances that were hard to imagine even a decade ago. It is possible that this approach eventually leads to exact learning, effectively by hill-climbing to a teaching set. Alternatively, this path may gradually eliminate the "embarrassing" mistakes made by current frontier models and lead to systems that appear logically coherent for all or most practical purposes, with no paradigm shift needed. 
However, as the models become more capable and societies grant them more agency, ensuring exactness becomes a safety concern. Even unlikely mistakes are unacceptable if they lead to a potentially catastrophic outcome. In addition, acknowledging that the goal is exact learning can help to speed up progress 
by clarifying what is important and changing the associated language,
even if the community stays on the path of manual and gradual changes.

\paragraph{Verifying exact learning is very hard} Verifying that the learner has obtained the exact solution is a difficult problem by itself, especially if we consider non-trivial models such as neural networks. Even on the simple problem of comparing two $m$-bit integers discussed earlier, we might need to evaluate the network on all $2^{2m}$ inputs to verify exactness.
An alternative is to develop formal verification methods,
perhaps starting from interpretability research
\citep{sharkey2025openproblemsmechanisticinterpretability}
and/or taking inspiration from the 
research done in developing certified adversarially robust methods  \citep[e.g.,][]{dvijotham2018dual,wong2018provable}.

The verification problem appears especially difficult in the case of natural language reasoning, where there are many different ways to express the same statement. Still, we believe that the situation is not hopeless. For example, language ambiguity can be removed by instructing models to produce output in a particular format. Beyond enumerating all input-output pairs, verification methods could possibly reuse ideas from the literature on systematic generalization where multiple distinct distributions are widely used for training and testing (the ``generalization split'', e.g., \citealt{Csordas2023-eb}),
and also from the adversarial learning literature \citep{szegedy2013intriguing,goodfellow2014explaining}, where
algorithms, that may have access to the models tested, are used to generate examples that may make the models fail. In addition, knowing whether exact learning was achieved is a different goal than designing algorithms that make this happen. That is, there is value in changing the goal to exact learning even if on a specific instance verifying whether learning succeeded completely is hard.
Finally, similarly to what was done in the adversarial robustness literature, one may design methods that also make certification easier \citep[e.g.,][]{carlini2023certified}.

\paragraph{Using symbolic inputs solves the problem}
Using learning to speed up reasoning (more generally, search) has been highly successful.
For this to work, the input to the system needs to be an appropriate ``symbolic'' form.
Examples include AlphaZero \citep{silver2018general}, 
and AlphaProof (\citealt{deepmind2024alphaproof}, without the autoformalization component).
When learning is used this way, the correctness of the whole system is guaranteed -- failure of learning may slow down the system, but it will not make it produce incorrect answers.
Unfortunately, this approach is limited to 
problems where inputs are available in, or are 
easy to transform to a symbolic form, such as in games, 
or formal mathematics.
However, when this is not the case (e.g., the input comes from sensors, or in the form of natural language text), the approach can only be applied when a separate system is learned to translate back and forth between the inputs, the symbolic representation and the outputs.

\paragraph{Using a symbolic backend with a learnt autoformalizer}
Prominent systems following this approach include 
\citep{deepmind2024alphaproof,lin2025goedel}.
While we acknowledge that AI systems can benefit from a hybrid system that combines search with pattern matching and that learning only parts of a system can increase the chance of success, 
at their core, these approaches still rely on a learning component that must perform flawlessly to ensure correct solutions. While the underlying language translation problem may be easier than full-blown reasoning (i.e., it spares a combinatorial search problem), this problem is still highly nontrivial.
As a case in point, \citet{lin2025goedel} reports 
an on-distribution correctness of autoformalization of $76\%$
and an off-distribution correctness of $46\%$ (cf. Table~7).
Other issues with this approach are that the language of the backend, if designed by humans, may be a poor choice that limits the performance of the whole system 
(aligned with the bitter lesson by \citet{sutton2019bitter})
and outside of the narrow domain of mathematics, 
and examples to train the autoformalizer may be even harder to come by.

\section{Conclusion}
\label{sec:conc}
We have argued that the current approach to develop AI systems, heavily reliant on the paradigm of statistical learning, 
is fundamentally misaligned with the demands of general intelligence. 
This misalignment is the most acute when the AI systems are applied in engineering, mathematics, or science, where flawless formal reasoning is paramount. 
We have proposed a pivot towards \emph{exact learning}, which aims for universal correctness. 
We argued that this universal correctness is already on the mind of many in the community.
We hope that following our proposal progress will be sped up 
as the community will be given a better language to describe the problems faced, 
not only raising awareness towards the different demands of exact and statistical learning,
but also facilitating the exploration of new ideas more aligned with the new criterion.
 Towards this goal, important  directions for future research include interactive learning (guided by the teacher, learner, or both) as well as changes to the learning task, such as teaching algorithms step-by-step by explaining the process rather than simply providing input-output pairs. 
Our inspiration is that the new perspective we promote will lead to similar changes as
what happened in other fields after a similar change of goals, one example of which is
the problem of estimation in the presence of heavy tailed data, which led to numerous important new results, including new algorithms
\citep{lugosi2019mean,loh2024theoretical}.

\bibliographystyle{plainnat}

\newpage
\appendix
\onecolumn

\section{Further Related Works}
\label{sec:related}
This paper touches upon several areas of related work, including statistical learning theory, neuro-symbolic AI, systematic generalization in neural networks, classical exact learning, adversarial examples, and out-of-distribution generalization. To avoid repetition, we focus on aspects not already discussed.

Learning to generalize without failure is known as the problem of \emph{systematic generalization} in deep learning \citep{lake2018generalization, bahdanau2018systematic}. Works in this area
 examine the ability of neural networks to generalize to novel combinations of seen concepts, often in the context of tasks that require compositional reasoning.
 While these studies typically focus on empirical evaluations of specific architectures and datasets, our goal is to initiate a systematic theoretical study to understand the challenges of exact learning. 

The challenges of exact learning are also related to the phenomenon of adversarial examples \citep{szegedy2013intriguing, goodfellow2014explaining}, where small, carefully chosen inputs can lead to misclassification, even in high-performing models (and, strikingly, often these inputs are obtained  by slightly changing correctly 
classified inputs, even in the physical world, see, e.g., \citealp{kurakin2017adversarial,hosseini2018semantic}).
These "unsatisfactory glitches" highlight the brittleness of neural networks and their sensitivity to inputs that deviate even slightly from the training distribution.
The main connection between our work and this areas is that exact learning would prevent adversarial examples. However, the main direction of research in this field became the problem of finding smooth predictors \citep{bubeck2023robustness}, and as such bears no direct relevance on our problem. Nevertheless, the underlying phenomenon, that is, the sensitivity of learned classifiers to the data distribution, is the same. 

Our findings are also related to the quest for models that have good out-of-distribution (OOD) generalization capabilities \citep{peters2017elements,arjovsky2019invariant}. In a way, our framework demands the ultimate extreme form of OOD generalization: learning exactly is equivalent to perfect generalization regardless of the target distribution. Our result that shows that for a benign distribution, such as the uniform distribution over a discrete domain, exponentially many examples may be needed for exact learning also gives a hard limit on OOD generalization abilities.

\section{Proofs}

\alllearnersfail*
\begin{proof}
For $h,h'\in \{0,1\}^{\cX}$, let 
\[
A_{h,h'} = \{x \in \cX \,:\, h(x) = h'(x) \}
\]
be the set of points in the domain $\cX$ 
where $h$ and $h'$ coincide, or \emph{agree}.
Now fix $\rho\in \cM_1(\cX)$, a probability distribution over $\cX$, 
that will be used to generate samples. 
Also, fix $h,h'\in \cH$, $h\ne h'$. 
Let $\phi_n(\cA,\cH,\rho)$ denote the probability that the learner will fail on the exact learning problem where $n$ i.i.d. data points are generated from $\rho$, labeled using some hypothesis in $\cH$:
\[
\phi_n(\cA,\cH,\rho) = \sup_{h\in \cH} \PP( \cA(D_h)\ne h )\,,
\]
where $D_h = ((X_1,h(X_1)),\dots,(X_n,h(X_n))$ and $X_1,\dots,X_n\in \cX$ is an i.i.d. sample from $\rho$. For brevity, drop the arguments of $\phi_n$: 
$\phi_n = \phi_n(\cA,\cH,\rho)$.

Now, let $h\ne h'$, $h,h'\in \cH$.
Note that $\cD_h=\cD_{h'}$ if and only if $X_1,\dots,X_n\in A_{h,h'}$.
Now, assume that $X_1,\dots,X_n\in A_{h,h'}$.
Hence, $\cA(\cD_h)=\cA(\cD_{h'})$.
Since $h\ne h'$, it follows that either
$\cA(\cD_h)\ne h$ or $\cA(\cD_{h'})\ne h'$ must hold.
Indeed, if $\cA(\cD_h)\ne h$ does not hold, then
$\cA(\cD_h)=h$. Then, $\cA(\cD_{h'}) = \cA(\cD_h)=h\ne h'$.
Putting things together, we have
\begin{align*}
(\rho(A_{h,h'}) )^n
& = \PP( X_1,\dots,X_n \in A_{h,h'} ) \\
& = \PP( \cD_h = \cD_{h'} ) \\
& \le \PP( \cA(\cD_h)\ne h \text{ or }\cA(\cD_{h'})\ne h' ) \\
& \le \PP( \cA(\cD_h)\ne h ) + \PP( \cA(\cD_{h'})\ne h' ) \\
& \le 2\phi_n \,.
\end{align*}
This essentially establishes the quantitive result we were looking for.

To get a bound that is easier to compute, we further lower bound $(\rho(A_{h,h'}))^n$. For this, we will use that $x \mapsto (1-x)^n$ is convex on $0\le x \le 1$ and hence lies above its first order approximation, $1-nx$, taken at $x=0$.
Hence, with $x = 1-\rho(A_{h,h'})$, we get 
$(\rho(A_{h,h'}))^n  \ge 1- n (1-\rho(A_{h,h'}))$.
Taking the supremum over $h\ne h'$, $h,h'\in \cH$ finishes the proof.
\end{proof}

\symlearnersfail*
\begin{proof}
Assume that $\cA$ is $G$-symmetric. Fix $\rho\in \cM_1(\cX)$ and $h\in \{0,1\}^{\cX}$
and let $\cH = \{ g h \,: g\in G\}$.
Denote the probability that $\cA$ fails on identifying a specific labeling function $h'$ when the input data comes from $\rho$ by $\phi_n(\cA,\rho,h')$.
For simplicity assume that if $X\sim \rho$ then $g X$ has the same distribution $\rho$, regardless of the choice of $g\in G$ (such distributions can be always found by symmetrizing any distribution, which is possible when $G$ is not too big).
By the $G$-symmetry of $\cA$, $\phi_n(\cA,g \rho, g h)=\phi_n(\cA,\rho,h)$ for any $g\in G$.
Since $\rho$ is $G$-symmetric, $\phi_n(\cA,g \rho, g h)=\phi_n(\cA,\rho,gh)$.
Hence, $\phi_n(\cA,\rho,h) = \phi_n(\cA,\rho,gh)$ for any $g\in G$.
By \cref{thm:dis}, $\sup_{g\in G} \phi_n(\cA,\rho,g h)\ge 1/4$ when $n\le 1/(2 \inf_{g\in G, g\ne \mathbf{1}_G} \PP( h(X)\ne (g h)(X))$ where $X\sim \rho$ and $\mathbf{1}_G$ is the identity of the group $G$.
However, since 
$\phi_n(\cA,\rho,h) = \phi_n(\cA,\rho,gh)$ for any $g\in G$, it follows that when
 $\sup_{g\in G} \phi_n(\cA,\rho,g h)\ge 1/4$ then 
  $\phi_n(\cA,\rho,h)  = \sup_{g\in G} \phi_n(\cA,\rho,g h)\ge 1/4$ also holds.
Thus, if $\cA$ is a learner that is $G$-symmetric then the critical sample size will be lower bounded by
$1/(2 \inf_{g\in G, g\ne \mathbf{1}_G} \PP( h(X)\ne (g h)(X))$, the reciprocal value of the smallest disagreement probability between $h$ and its transformation $gh$ where $g$ ranges through the elements of $G$ except for the identity element of $G$.
\end{proof}

In the text, we made the following claim:  

\begin{proposition}\label{prop:supp}
Suppose that $f \in \cFlin(\cX)$, a linear classifier with finite domain $\cX\subset \R^d$, i.e., $f(x)=\sgn(\ip{w,x}) + b$ for some $w\in \R^d$ and $b \in \R$. 
Then there exists a dataset $\cD$ with $|\cD| \leq 2d+2$ such that the maximum-margin classifier on $\cD$
is equal to $f$.
\end{proposition}
\begin{remark}
Let $ w_\star(\cD)$ denote the weights of the maximum margin classifier when it is fed with data $\cD$.
Let $\cD_0 = \{(x, f(x)) : x \in \cX\}$ be the dataset associated with the entire input space. The vectors $\{x : (x, y) \in \cD\}$ constructed in the proof of \cref{prop:supp} are support vectors associated with $w_\star(\cD_0)$ in the sense that $f(x) \ip{x, w_\star(\cD_0)} = 1$. There can be an enormous number of support vectors and the set $\cD$ is a carefully chosen subset.
On the other hand, when $\cX = \R^d$,
\citet{liu2016teaching}
show that a teaching set of a single element always exist (due to the homogenity of the separating hyperplane, although the example in the teaching set may not be a genuine data point). While interesting, this is not suitable for applications where the domain is inherently discrete, as is our case when we consider reasoning tasks.
\end{remark}

\begin{proof}[Proof of \cref{prop:supp}]
Given a dataset $\cD$ let $\cD^\delta = \{x - x' : x,x' \in \cD, f(x) = 1, f(x') = -1\}$. Let 
\begin{align*}
P(\cD) = \{w \in \R^d : \ip{\delta, w} \geq 2 \,, \forall \delta \in \cD^\delta\}
\end{align*}
Since $f \in \cFlin(\cX)$, the set $P(\cD_0)$ is non-empty.
By definition,
\begin{align*}
    w_\star(\cD_0) = \argmin_{w \in P(\cD_0)} \frac{1}{2} \norm{w}^2 \,.
\end{align*}
By the definition of $P(\cD_0)$, $w_\star(\cD_0) \neq \zeros$. 
By the first-order optimality conditions, for all $w \in P(\cD_0)$,
\begin{align*}
    \ip{w_\star(\cD_0), w - w_\star(\cD_0)} \geq 0 \,.
\end{align*}
Equivalently, $-w_\star(\cD_0)$ is in the normal cone of $P(\cD_0)$ at $w_\star(\cD_0)$.
Since $\cX$ is finite the normal cone at $w_\star(\cD_0)$ is
\begin{align*}
\cone(N) \text{ with } N = \{-\delta : \delta \in \cD_0^\delta, \ip{\delta, w_\star(\cD_0)} = 2 \}\,.
\end{align*}
Hence there exists a function $\lambda : N \to [0, \infty)$ such that
\begin{align*}
-w_\star(\cD_0) = \sum_{\eta \in N} \lambda(\eta) \eta \,.
\end{align*}
Moreover, by Carath\'eodory's theorem \citep[Theorem 17.1]{Roc70}, $\lambda$ can be chosen to be supported on not more than $d+1$ elements of $N$.
Since each element in $N$ is the difference of two vectors in $\cD_0$ there exists a set $\cD$ with size at most $2d+2$ such that $-N \subset \cD^\delta$.
But then $-w_\star(\cD_0)$ is in the normal cone of $P(\cD)$ at $w_\star(\cD_0)$ and hence
\begin{align*}
    w_\star(\cD_0) = \argmin_{w \in P(\cD)} \frac{1}{2} \norm{w}^2 = w_\star(\cD) \,.
\end{align*}
Hence $w_\star(\cD) = w_\star(\cD_0)$.
The result follows since $\hat f(\cdot | w_\star(\cD_0)) = f$ by definition.
\end{proof}

\section{Experiment details}
\label{app:experiments}

\paragraph{Architecture and training.} Our implementation of transformer training and evaluation builds upon the NanoDO framework \citep{nanodo}.  For the simple logic experiments, we train a decoder-only transformer with 10 layers and embedding size 1024 split across 8 heads. We use MLPs with hidden size 1024 and GELU activations \citep{hendrycks2016gaussian}. We use RoPE positional embeddings \citep{su2024roformer}.  We train using the AdamW optimizer \citep{loshchilov2017decoupled} with batch size 256, peak learning rate 0.0001, warmup and final rate of 0.00001, and 2000 warmup steps. We clip gradients by global norm 1. We use a custom tokenizer designed for logic data. The vocabulary includes the 150 zeroth order predicates used in \citet{zhang2023paradox} and the following tokens: \{'Facts:', 
    'Rules:',
    'Query:',
    'Answer:',
    'Reasoning:',
    'Newfact:',
    'yes',
    'no',
    'and',
    'Is',
    'Since',
    'If',
    'then',
    '?',
    '.',
    ',',
    'No\_other\_facts\_can\_be\_proven'\}.

\paragraph{Data.} We generated propositional logic problems using the recipes described in \citet{zhang2023paradox}. RP problems are generated by randomly sampling a subset of predicates, and randomly sampling facts and rules using those predicates. The facts and rules are then used to compute the truth values, and a query is selected at random. LP problems are generated by first assigning truth values to a subset of predicates, and then sampling facts and rules consistent with those values. 
To ensure balanced labels, the query was selected uniformly at random from the true predicates with probability 0.5.
The problems we generated had between 5 and 20 predicates, between 0 and 40 rules, and reasoning depth (number of forward chaining steps) up to 6. Rather than using a fixed dataset, we continuously generated data on-the-fly. We generated reasoning traces of the form in Figure~\ref{fig:simple_logic} by running forward-chaining until the query was proven true, or until no further predicates could be proven true. Rather than predicting all tokens, we used conditional training where we only predicted the tokens following the prompt, i.e. following the '?' token.

\paragraph{Evaluation.} We evaluate the trained models on a subset of 2048 randomly generated RP and LP problems, using greedy decoding.

\end{document}